\documentclass{ieeeaccess}
\usepackage[utf8]{inputenc}
\usepackage[numbers]{natbib}
\usepackage{amsmath}
\usepackage{amssymb}
\usepackage{amsthm}
\usepackage{algorithm}
\usepackage[noend]{algpseudocode}
\usepackage{booktabs}
\usepackage{amsfonts}
\usepackage{multicol}
\usepackage{graphicx}
\usepackage{hanging}
\usepackage{textcomp}
\usepackage{url}
\def\BibTeX{{\rm B\kern-.05em{\sc i\kern-.025em b}\kern-.08em
    T\kern-.1667em\lower.7ex\hbox{E}\kern-.125emX}}

\title{p2pGNN: A Decentralized Graph Neural Network for Node Classification in Peer-to-Peer Networks}
\newtheorem{theorem}{Theorem}

\begin{document}
\author{\uppercase{Emmanouil Krasanakis}, \uppercase{Symeon Papadopoulos}, \uppercase{Ioannis Kompatsiaris}}
\address{Centre for Research \& Technology Hellas, Information Technologies Institute, 6th km Charilaou-Thermi, 57001, Thermi, Thessaloniki, Greece \\e-mail: \{maniospas,papadop,ikom\}@iti.gr}
\tfootnote{This work was partially funded by the European Commission under contract numbers H2020-951911 AI4Media and H2020-825585 HELIOS.}
\corresp{Corresponding author: Emmanouil Krasanakis (e-mail: maniospas@iti.gr).}
\history{March 2022}
\doi{10.1109/ACCESS.2022.3159688}
\begin{abstract}
    In this work, we aim to classify nodes of unstructured peer-to-peer networks with communication uncertainty, such as users of decentralized social networks. Graph Neural Networks (GNNs) are known to improve the accuracy of simple classifiers in centralized settings by leveraging naturally occurring network links, but graph convolutional layers are challenging to implement in decentralized settings when node neighbors are not constantly available. We address this problem by employing decoupled GNNs, where base classifier predictions and errors are diffused through graphs after training. For these, we deploy pre-trained and gossip-trained base classifiers and implement peer-to-peer graph diffusion under communication uncertainty. In particular, we develop an asynchronous decentralized formulation of diffusion that converges to centralized predictions in distribution and linearly with respect to communication rates. We experiment on three real-world graphs with node features and labels and simulate peer-to-peer networks with uniformly random communication frequencies; given a portion of known labels, our decentralized graph diffusion achieves comparable accuracy to centralized GNNs with minimal communication overhead (less than 3\% of what gossip training already adds).
\end{abstract}


\titlepgskip=-15pt
\maketitle

\section{Introduction}\label{introduction}
\PARstart{T}{he}  pervasive integration of mobile devices and the Internet-of-Things in everyday life has created an expanding interest in processing their collected data \cite{wang2018deep,li2018learning,lim2020federated}. However, traditional data mining techniques require communication, storage and processing resources proportional to the number of devices and raise data control and privacy concerns. An emerging alternative is to mine data at the devices gathering them with protocols that do not require costly or untrustworthy central infrastructure. One such protocol is \emph{gossip averaging} \cite{daily2018gossipgrad}, which averages local model parameters across pairs of devices during training.

As an example, existing social media applications often rely on central platforms, such as Meta (Facebook, Instagram), Viber, and Telegram. However, increasing concerns of how personal and potentially sensitive data are handled by central controllers have motivated the development of decentralized social media \cite{seong2010prpl}, in which user devices communicate directly with each other. Thus, new opportunities are created for AI-powered decentralized media subsystems. Yet, to date, there is a lack of machine learning frameworks to support decentralized machine learning in uncontrolled communication environments. In this work, we make steps towards the development of such frameworks for deployment of decentralized graph-based learning ``in the wild''.

\par
We tackle the specific problem of classifying points of a shared feature space when each one is stored at the device generating it, i.e. each device accesses only its own point but all devices collect the same features. For example, mobile devices of decentralized social media users could predict user interests based on locally stored content features, such as the bag-of-words of posted messages, and user-disclosed interests as target labels. We further consider devices that are nodes of peer-to-peer networks and communicate with each other based on underlying relations, such as friendship or proximity. In this setting, social network overlays coincide with communication networks. However, social behavior dynamics (e.g. users going online or offline) could prevent devices from communicating on-demand or at regular intervals. Ultimately, with who and when communication takes place can not be controlled by learning algorithms.
\par
When network nodes corresponding to data points are linked based on social relations, a lot of information is encapsulated in their link structure in addition to data features. For instance, social network nodes one hop away often assume similar classification labels, a concept known as homophily \cite{mcpherson2001birds,berry2020going}. Yet, existing decentralized learning algorithms do not naturally account for information encapsulated in links, as they are designed for structured networks of artificially generated topologies. In particular, decentralized learning often focuses on creating custom communication topologies that optimize some aspect of learning \cite{koloskova2019decentralized} and are thus independent from data. 
\par
Our investigation differs from the above assumption in that we classify data points stored in decentralized devices that form unstructured communication networks, where links capture real-world activity (e.g. social interactions) unknown at algorithm design time and thus of irregular structure, as demonstrated in Fig.~\ref{fig:topologies}. In our setting, devices coincide with graph nodes and we use the two terms interchangeably.

\begin{figure}[htpb]
    \centering
    \frame{\includegraphics[width=0.2\textwidth,trim={4cm 3cm 3cm 4cm},clip]{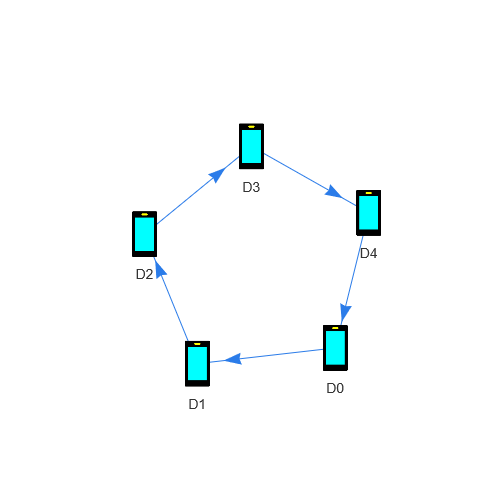}}
    \frame{\includegraphics[width=0.2\textwidth,trim={4cm 4cm 4cm 4cm},clip]{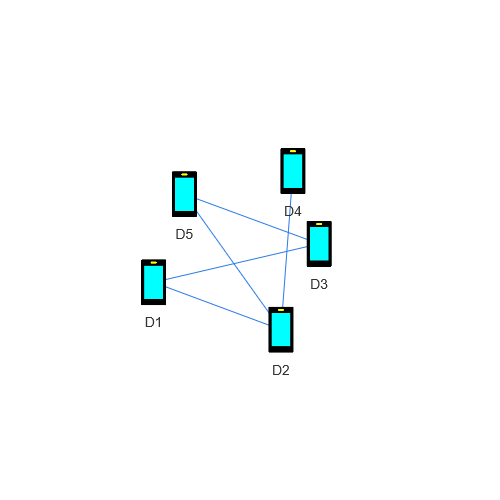}}
    \caption{A directed ring-structured communication network (left) and an undirected unstructured example (right) of devices $D1,\dots,D5$.
    }
    \label{fig:topologies}
\end{figure}
 
\par
If a centralized service performed classification, Graph Neural Networks (GNNs) could be used to improve the accuracy of base classifiers, such as ones trained with gossip averaging, by accounting for link structure (Subsection~\ref{gnns}). But, if we tried to implement GNNs with the same decentralized protocols, connectivity constraints would prevent devices from timely collecting latent representations from communication neighbors, where these representations are needed to compute graph convolutions. 
\par
To tackle this problem, we propose working with \textit{decoupled GNNs}, where network convolutions are separated from base classifier training and organized into graph diffusion components. Given this architecture, we start from either pre-trained base classifiers or train those with gossip protocols. We then realize graph diffusion in peer-to-peer networks by developing an algorithm, called \textit{p2pGNN}, whose fragments run on each node and converge at similar predictions as graph diffusion while working under uncontrolled irregular communication initiated by device users. Our analysis is supported by a novel theoretical construct, which we dub \textit{decentralized graph signals}, that describes decentralized diffusion primitives in the irregular communication setting. Critically, our algorithm supports online modification of base classifier predictions while these are being diffused. As a result, all components of decoupled GNN fragments run at the same time and eventually converge to the desired results.
\par
Our contribution is twofold. First, we establish a decentralized setting for classifying peer-to-peer network devices. To our knowledge, our approach is the first that considers communication links themselves useful for the decentralized learning task, i.e. in networks where communication topology is retrieved from the real world instead of being imposed on it. We also introduce the concept of decentralized graph signals that formalize graph diffusion in this setting.
\par
Second, we develop the p2pGNN algorithm that parses decentralized graph signals and, given existing methods of training or deploying base classifiers in peer-to-peer networks under uncertain availability, approximates originally centralized decoupled GNN components to improve accuracy. For this algorithm, we theoretically show fast convergence to similar prediction quality as centralized architectures. Furthermore, we experiment on simulated peer-to-peer networks under uncertain availability, where we verify that it successfully takes advantage of graph diffusion components to improve base classifier accuracy, closely matches the accuracy of fully centralized computations, and incurs only small communication overheads.

\section{Background}\label{background}
\subsection{Graph Neural Networks}\label{gnns}
Graph Neural Networks (GNNs) are a machine learning paradigm in which links between data samples are used to improve the predictions of base neural network models \cite{wu2020comprehensive}. In detail, samples are linked to form graphs based on real-world relations and information diffusion schemes smooth (e.g. average) latent attributes across graph neighbors before transforming them with dense layers and non-linear activations to new representations to be smoothed again. This is repeated either ad infinitum or for a fixed number of steps to combine original representations with structural information. 
\par
Notably, in our setting, there is an 1-1 correspondence between samples and devices. However, although GNN propagation takes place in a decentralized-like manner, i.e. nodes work independently, transformation parameters are shared and learned across all nodes.
\par
GNN architectures tend to suffer from over-smoothing if too many (e.g. more than two) smoothing layers are employed. However, using few layers limits architectures to propagating information only few hops away from its original nodes. Mitigating this issue often involves recurrent links to the first latent representations, which lets GNNs achieve at least the same theoretical expressiveness as graph filters \cite{klicpera2018predict,chen2020simple}. In fact, it has been argued that the success of GNNs can in large part be attributed to the use of recurrency rather than end-to-end training of seamless architectures \cite{huang2020combining}. As a result, recent works have introduced decoupled architectures that achieve the same theoretical expressive power as end-to-end training by training base statistical models, such as two-layer perceptrons, to make predictions, and smoothing the latter through graph edges.
\par
In this work, we build on the FDiff-scale prediction smoothing proposed by Huang et al. \cite{huang2020combining}, which diffuses the base predictions and respective errors of base classifiers to all graph nodes using a constrained personalized PageRank that retains training node labels. Then, a linear trade-off between errors and predictions is calculated for each node and the outcome is again diffused with personalized PageRank to make final predictions. This process generalizes to multiclass predictions by replacing node values propagated by personalized PageRank with vectors holding prediction scores, where initial predictions are trained by the base classifier to minimize a cross-entropy loss. Architecture details and our motivation for using it are discussed in Subsection~\ref{stateppr}.

\subsection{Decentralized Learning}\label{decentralized learning}
Decentralized learning refers to protocols that help pools of devices learn statistical models by accounting for each other's data. Conceptually, each device holds its own autonomous version of the model 
and training aims to collectively make those converge to being similar to each other and to a centralized training equivalent, i.e. to be able to replicate would-be centralized predictions locally.
\par
Many decentralized learning practices have evolved from distributed learning, which aims to speed up the time needed to train statistical models by splitting calculations among many available devices, called workers. Typically, workers perform computationally heavy operations, such as gradient estimation for subsets of training data, and send these to a central infrastructure that orchestrates the learning process. 
\par
A well-known variation of distributed learning occurs when data batches are split across workers a-priori, for example because they are gathered by these, and are sensitive in the sense that they cannot be directly presented to the orchestrating service. This paradigm is called federated learning and is often realized with the popular federated averaging (FedAvg) algorithm \cite{mcmahan2017communication}. FedAvg performs several learning epochs in each worker before sending parameter gradients to a server that uses the average across workers to update a model and send it back to all of them.
\par
By definition, distributed and federated learning train one central model that is fed back to workers to make inferences. However, gathering gradients and sending back the model requires a central service with significantly higher throughput than individual workers to simultaneously communicate with all of them and orchestrate learning. To reduce the related infrastructure costs and remove the need for a central authority, decentralized protocols have been introduced to let workers directly communicate with each other.\footnote{Decentralized learning is sometimes referred to as decentralized federated learning, but this is different than distributed federated learning.} These require either constant communication between workers or a rigid (e.g. ring-like) topology and many communication rounds to efficiently learn  \cite{lian2018asynchronous,luo2020prague,zhou2020communication}. Most decentralized learning practices have evolved to or are variations of \textit{gossip averaging}, where devices exchange and average (parts of) their learned parameters with random others \cite{daily2018gossipgrad,hu2019decentralized,savazzi2020federated,hegedHus2021decentralized,danner2018token,koloskova2019decentralized}.

\section{A Peer-to-Peer Graph Neural Network}\label{our approach}
\subsection{Problem Formulation}
We work on peer-to-peer networks whose devices are linked based on their ability to send messages to each other, even through channels of uncertain availability. These networks can be described with static adjacency matrices $A\in\mathbb{R}^{N\times N}$, where $N$ is the number of nodes in the peer-to-peer network. These matrices comprise elements:$$A[u,v]=\{1\text{ if }(u,v)\text{ are linked}, \text{else }0\}$$
\par
We further consider devices $u$ to hold feature vectors $X[u]\in\mathbb{R}^F$ of a shared feature space, such as average word embeddings of user text messages, where $F$ is the number of features and it is the same for all nodes. Finally, some training devices in the network $u\in V_{train}$ hold manually provided class labels with one-hot encodings $Y[u]\in\mathbb{R}^C$, where $C$ is the number of classes and $\text{argmax}{Y[u]}$ retrieves labels from their encodings. We aim to make encoding predictions $\hat{Y}[u]$ for \textit{all} devices $u$ so that $\text{argmax}{\hat{Y}[u]}$ correspond to true class labels with high accuracy. Importantly, to avoid centralization, each device needs to create predictions about itself, for instance to estimate its user's interests among a list of topics, while only viewing information transmitted by communicating devices.
\par
If our goal was to make feature-based predictions without accounting for communication links, we would select a base classifier $R_\theta:\mathbb{R}^F\to\mathbb{R}^C$ of trainable parameters $\theta$ and deploy its computational model to all devices. Then, devices would learn their own parameters $\theta[u]$ to make predictions $R_{\theta[u]}(X[u])$ that tightly approximate centralized optimization of the computational model's parameters on training node labels. For example, Gossip averaging would set up an iterative process performing gradient updates on local data in training nodes while averaging parameters between communicating nodes. This would allow node predictions to account for training data residing more than one hops away.
\par
If we had the luxury of a central service and willingness of training device users (only) to disclose their data to it, we could instead perform centralized training and deploy a common set of learned parameters through the service. This way, non-training devices would classify themselves without exposing local data.
\par
In Section~\ref{introduction} we argued that communication links often pertain to real-world relations, which GNNs can leverage to improve classification accuracy. In a centralized setting, this could be achieved with GNN classifiers $G_\theta:\mathbb{R}^{F\times N}\times \mathbb{R}^{N\times N}\to\mathbb{R}^C$ of parameters $\theta$. These take two inputs: a) tables gathering all node features $X\in\mathbb{R}^{N\times F}$, where rows $X[u]$ are device $u$ feature vectors, and b) adjacency matrices $A$. They output prediction matrices $\hat{Y}=G_\theta(X, A)$, whose rows $\hat{Y}[u]$ are device $u$ prediction vectors. Unfortunately, even if parameters $\theta$ where to be learned by a centralized service, these classifiers could not be directly deployed to perform in-device inference, since they rely on non-local information, such as the whole communication network's structure and features of nodes more than one hops away propagated through the structure. 
\par
In this work, our goal is to develop GNNs that let peer-to-peer devices $u$ classify themselves by running fragments $\text{frag}_v(G_{\theta[u]}, X[u], A[u])$ of GNN architectures $\mathcal{G}_\theta$. These only account for local features $X[u]$ and only communicate with the fragments of linked neighbors found in $A[u]$. Then, fragments in devices both learn parameters $\theta[u]$ that approximate optimal ones and perform additional computations that let them approximate centralized estimations:
$$\hat{Y}[u]\approx \text{frag}_v(G_{\theta[u]}, X[u], A[u])$$

\subsection{Communication Protocol}\label{protocol}
Peer-to-peer networks often suffer from node churn, power usage constraints, as well as virtual or physical device mobility that cumulatively make communication channels between nodes irregularly-available. In this work, we assume that linked nodes keep communicating over time without removing links or introducing new ones, though links can become temporarily inactive. We expect this assumption to hold true in social networks that evolve slowly, i.e., in which user interactions are many times more frequent than link changes. From the perspective of link mining, these networks can be viewed as static relational graphs. 
\par
We stress that static relations are exhibited even when devices rapidly switch communication patterns, as long as they are limited within fixed sets of neighbors. In practice, slow evolution can even be enforced by narrowing our focus to communication between long-time social neighbors, therefore ignoring temporal social behavior noise.
\par
Thus, we consider static adjacency matrices $A$ like above and encode uncertainty with time-evolving communication matrices $A_{com}(t)$, whose non-zero elements indicate exchanges through the corresponding links: $$A_{com}(t)[u,v]=\{1\text{ if }u,v\text{ communicate at time }t, \text{else }0\}$$
To simplify the rest of our analysis, and without loss of generality, we adopt a discrete notion of time $t=0,1,2,\dots$ that orders the sequence of communication events. We stress that real-world time intervals between consecutive timestamps could vary and that, for the communication adjacency matrix $A$, it holds that $A_{com}(t)[u,v]=1\Rightarrow A[u,v]=1$.
\par
We now provide a framework in which peer-to-peer nodes learn to classify themselves by exchanging information through channels represented by time-evolving communication matrices. This waits for the infrequent timeframes when channels become active and executes the broadly popular Send-Receive-Acknowledge communication protocol to exchange information. In particular, devices $u$ are equipped with identifiers $u$.id and operations $u$.{}{SEND}, $u$.{}{RECEIVE} and $u$.{}{ACKNOWLEDGE} that respectively implement message generation, receiving message callbacks that generate new messages to send back, and acknowledging that sent messages have been received while sending back the recipient's generated messages. Expected usage of these operations is demonstrated in Algorithm~\ref{protocolalg}. 

\begin{algorithm}[htpb]
\begin{algorithmic}
\State \textbf{Inputs:} devices $u\in V$ with identifiers $u.$id, time-evolving $A_{com}(t):V\times V\to\mathbb{R}$
\For{$t= 0,1,2,\dots$}
\ForAll{$(u,v)$ such that $A_{com}(t)[u,v]=1$}
\State message$\leftarrow$ $u$.\textsc{send}($v$.id)
\State reply$\leftarrow$ $v$.\textsc{receive}($u$.id, message)
\State $u$.\textsc{acknowledge}($v$.id, reply)
\EndFor
\EndFor
\end{algorithmic}
\caption{Send-Receive-Acknowledge protocol}\label{protocolalg}
\end{algorithm}

\subsection{GNN Architecture}\label{stateppr}
GNN architectures can be used to combine relation-based peer-to-peer connectivity with device features to improve classification accuracy compared to classifiers using only features. This is achieved by incorporating graph convolutions in multilayer parameter-based neural network transformations to smooth latent representations across neighbor nodes. 
\par
We identify two realistic implementations of smoothing in peer-to-peer networks under uncertain availability: either a) the last retrieved representations are used, or b) node features and links from many hops away are stored locally for in-device computation of graph convolutions. In the first case, convergence to equivalent centralized model parameters is slow, since learning impacts neighbor representations only during communication.\footnote{Porting decentralized learning protocols designed for constant device availability in our setting requires waiting for other devices to send representations before running local computations. The option of using the last retrieved representations relaxes this scheme by progressing computations even when some neighbors take too long to communicate.} In the second case, multilayer architectures aiming to broaden node receptive fields from many hops away end up storing most network links and node features in each node; this violates data privacy and could be computationally intractable given limited device capabilities.
\par
To avoid these shortcomings, we build on existing decoupled GNNs outlined in Subsection~\ref{gnns}, which in our setting separate the challenges of training base classifiers with leveraging network links to improve predictions. In particular, they consider base classifiers that can parse features matrices $X$ to output matrices $R_\theta(X)$ with rows holding the predictions of respective feature rows $R_\theta(X)[u]=R_\theta(X[u])$. If base classifiers are trained on the features and labels of node sets $V_{train}$, we build on the FDiff-scale decoupled GNN's description \cite{huang2020combining}, whose predictions we transcribe as:
\begin{equation}\label{target diffusion}
\begin{split}
   \hat{Y}&= (\mathcal{I}-\beta D^{-0.5}AD^{-0.5})^{-1}P_\beta\big(R_\theta(X) 
   \\&\quad\quad+ (1-\beta)s (\mathcal{I}-A_{mask}D^{-1})^{-1}P_1 (Y-R_\theta(X)) \big)
\end{split}
\end{equation}
where $\mathcal{I}$ is the unit matrix, $D=\text{diag}([\sum_v A[u,v]]_{u})$ is a diagonal matrix of node degrees, masked adjacency matrices prevent diffusion from affecting training nodes with elements \begin{equation*}
    A_{mask}[u,v]=
    \{A[u,v]\text{ if }u=v\text{ or }v\not\in V_{train},
    \text{else } 0\},
\end{equation*} 
and a diagonal matrix is used to control the injection of personalized node information in the diffusion scheme per:
$$P_\gamma=\tfrac{1}{1-\beta}\text{diag}(\{1-\beta\text{ if }u\in V_{train}, \text{else }1-\gamma\}_u)$$
The values $\beta\in[0,1)$ (symbol chosen for clarity), $s\in\mathbb{R}$ are hyperparameters, whereas $\gamma$ is a variable that helps express the two versions of $P_\gamma$ with one formula. 
\par
In terms of our problem formulation, \eqref{target diffusion} effectively implements a GNN architecture $$G_\theta(X,A)=\text{diff}(R_\theta,X,Y,A)$$ that diffuses predictions $R_\theta(X)$ through the graph with an operation $\text{diff}(\dots)$. This comprises two sub-operations of the following form: \begin{equation}\label{ppr}
    \pi_\infty=(\mathcal{I}-aD^{-d}A_{mask}D^{d-1})^{-1}P_\gamma\pi_0
\end{equation}
where $a,d$ are again helper variables to express both sub-operations with one formula.
\par
The sub-operation performed first, i.e. the one inside the largest parenthesis block in \eqref{target diffusion}, is identical to \eqref{ppr} for $d=0,a=1,\gamma=1$. The last value makes it so that only the personalization $\pi_0[u]$ of training nodes $u\in V_{train}$ is diffused through the graph. The second sub-operation sets $d=0.5,a=\beta,\gamma=\beta$ and is equivalent to constraining the personalized PageRank scheme \cite{page1999pagerank,tong2006fast} with normalized communication matrix $D^{-d}AD^{d-1}$ so that it preserves original node predictions $\pi_n[u]=\pi_0[u]$ assigned to training nodes $v\in V_{train}$. Effectively, it is equivalent to restoring training node scores after each power method iteration $$\pi_{n+1}=\beta D^{-d}AD^{d-1}\pi_n+(1-\beta)\pi_0$$ where each iteration step is a specific type of graph convolution. The representations to be diffused by the two sub-operations are training node errors and a trade-off between diffused errors and node predictions respectively.

We stress that, although the above-described architecture exists in the literature, supporting its diffusion operation in peer-to-peer networks under uncertain availability requires the analysis we present in the rest of this section.

\subsection{Peer-to-Peer Personalized PageRank}\label{decppr}
If matrix row additions are atomic node operations, implementing the graph diffusion of \eqref{target diffusion} in peer-to-peer networks with uncertain availability is reduced to implementing the two versions of \eqref{ppr}'s constrained personalized PageRank presented above.
\par
Previous works have computed non-personalized (for which $\pi_0$ columns are normalized vectors of ones) or personalized PageRank in peer-to-peer networks by letting peers hold fragments of the network spanning multiple nodes and merging these when peers communicate \cite{parreira2006efficient,zhang2020decentralized,bahmani2010fast,hu2012localized}. Our setting is different in that peers coincide with nodes and merging network fragments requires untenable bandwidths proportional to network size to exchange merged sub-networks. Instead, we devise a new computational scheme that is lightweight in terms of communication. 
\par
On the surface, iterative synchronized convolutions require node neighbor representations at intermediate steps. However, an early work by Lubachevsky and Mitra \cite{lubachevsky1986chaotic} showed that, for non-personalized PageRank, decentralized schemes holding local estimations of earlier-computed node scores (or, in the case of graph diffusion, vectors) converge to the same point as centralized ones as long as communication intervals are bounded. 
\par
This motivates us to similarly iterate personalized PageRank by using the last communicated neighbor representations to update local nodes. In this subsection we mathematically describe this scheme and show that it converges in probability to the same point as its centralized equivalent with linear rate (which corresponds to an exponentially degrading error) and even if personalization evolves over time but still converges with linear rate. Notably, keeping older representations to calculate graph convolutions was not viable when these were entangled with representation transformations, but employing decoupled GNNs lets us separate learning from diffusion.
\par
To set up a decentralized implementation of personalized PageRank, we introduce a theoretical construct we dub \textit{decentralized graph signals} that describes decentralized operations in peer-to-peer networks while accounting for personalization updates over time, in case these are trained while being diffused. Our structure is defined as matrices $S\in\big(\mathbb{R}^C\big)^{N\times N}$ with multidimensional vector elements $S[u,v]\in \mathbb{R}^C$ (in our case $C$ is the number of classes) that hold in devices $u$ the estimate of device $v$ representations. Rows $S[u]$ are stored on devices $u$ and only cross-column operations are impacted by communication constraints.
\par
We now consider a scheme that updates decentralized graph signals $S(t)$ at times $t$ per the rules:
\begin{equation}\label{system}
    \begin{split}
        S(t)[u,v]&=S(t-1)[u,v]
        \\&\quad+A_{com}(t)[u,v]\big(\tfrac{S(t-1)[v][v]}{{D[u,u]^{d}}}-S(t-1)[u,v]\big)\\
        S(t)[u,u]&=P_\gamma[u,u]S_0(t)[u] + a \sum_v \tfrac{A_{mask}[u,v]}{D[u,u]^{1-d}}S(t)[v][v]
    \end{split}
\end{equation}
where $S_0(t)[u]\in\mathbb{R}^C$ are time-evolving representations of nodes $u$. The first of the above equations describes node representation exchanges between devices based on the communication matrix, whereas the second one performs a local update of personalized PageRank estimation given the last updated neighbor estimation that involves only data stored on devices $u$. Then, Theorem~\ref{thm:convergence} shows that the main diagonal of the decentralized graph signal deviates from the desired node representations with an error that converges to zero mean with linear rate. This weak convergence may not perfectly match centralized diffusion. However, it still guarantees that the outcomes of the two correlate in large part.

\begin{theorem}\label{thm:convergence}
Let $\lim_{t\to\infty}S_0(t)[u]= \pi_0[u]$ be bounded and converge in distribution with linear rate, the elements of $A_{com}$ be independent discrete random variables with fixed means with at least one of them less than $1$, $d\in\{0,0.5\}$, and either $a\in[0,1)$ or $a=1,d=0$. Then $\lim_{t\to\infty}S(t)[u,u]$ converges in distribution to $\pi_\infty[u]$ of \eqref{ppr} with linear rate.
\end{theorem}
\begin{proof}
Without loss of generality, we assume $V_{train}=\emptyset$, for which $A_{mask}=A$. More training nodes only add constraints to the diffusion scheme that force it to converge faster.
\par
Let $s(t)$ and $s_0(t)$ be vectors with elements $s(t)[u]=\mathbb{E}\{S(t)[u,u]\}$ and $s_0(t)[u]=\mathbb{E}\{P_\gamma[u,u]S_0(t)[u]\}$, where $\mathbb{E}\{\cdot\}$ is the expected value operation. Since the communication rate mean is fixed for each edge, it holds that:
$$s(\infty)=s_0(\infty)(1-a)+aD^{-d}AD^{d-1}s(\infty)$$
which, for a communication matrix $A$ and $a\in[0,1)$ yields the solution
$s(\infty)=(1-a)(I-aD^{-d}AD^{d-1})^{-1}s_0(\infty)$. For eigenvalues $\lambda$ of $D^{-d}AD^{d-1}$ when $d\in\{0,0.5\}$ it holds that $|\lambda|\leq 1$ (from the properties of doubly stochastic and Markovian matrices) and the corresponding eigenvalues of $D^{-d}AD^{d-1}$ become $1-a\lambda>0$, which makes it invertible. Hence, the solution is unique and coincides with $\pi_\infty$. For $a=1$ and $d=0$, $s_0(\infty)=\pi_\infty$ as the convergence point of the same irreducible Markov chain.
\par
For the same quantities, the convergence rate would be the same or faster if all communications took place with probability $p_{com}=\min_{u,v}\mathbb{E}\{A_{com}[u,v]\}< 1$ where $A_{com}$ is the communication matrix. Thus, we consider a communication matrix $A_{com}$ whose non-zero elements are sampled from $A$ with probability $p_{com}$ and analyse the latter to find the slowest possible convergence rate. In this setting, we obtain the recursive formula:
$$s(t)=s_0(t)(1-a)+aWs(t-1)$$
where $W=\mathbb{E}\{D^{-d}A_{com}D^{d-1}\}=p_{com}D^{-d}AD^{d-1}$. Thus, denoting as $\sigma=p_{com}\sigma_A$ the spectral radius of $W$, where $\sigma_A\leq 1$ is the spectral radius of the matrix $D^{-d}AD^{d-1}$ it holds that:
\begin{align*}
&\|s(t)-s(\infty)\|
\\&\quad=\|(1-a)(s_0(t)-s_0(\infty))+aW(s(t-1)-s(\infty))\|
\\&\quad\leq (1-a)\|s_0(t)-s_0(\infty)\|+a\sigma\|s(t-1)-s(\infty)\|
\\&\quad\leq (1-a)r_0^t\|s_0(0)-s_0(\infty)\|+a\sigma\|s(t-1)-s(\infty)\|
\end{align*}
where $r_0< 1$ is the linear convergence rate of $s_0(t)$. Thus, for $\sigma\leq p_{com}< 1,a\leq 1$, we calculate the behavior as $t\to\infty$ to obtain the linear convergence rate
$\lim_{t\to\infty}\frac{\|s(t)-s(\infty)\|}{\|s(t-1)-s(\infty)\|}\leq a \sigma<1$.
\end{proof}
\par
Algorithm~\ref{procedures}, which we call \textit{p2pGNN}, realizes \eqref{target diffusion} as decentralized algorithm fragments. These run on peer-to-peer network nodes $u$ and communicate with social neighbors $v$ under the Send-Receive-Acknowledge protocol to refine feature-based predictions based social communication links, as shown in Fig.~\ref{fig:overview}. We implement the protocol's operations, node initialization given prediction vectors and target labels, and the ability to update predictions. Nodes are initialized per $u.${INITIALIZE}($R_{\theta[u]}(X[u])$, $Y[u]$), where the last argument is a vector of zeroes for non-training nodes. The first argument is base classifier estimations from (locally) trained parameters $\theta[u]$ that can also be updated later on, for example after gossip averaging updates, by calling $u.${UPDATE}($R_\theta(X[u]))$. 
\par
We implement graph diffusion with decentralized graph signals \textit{predictions} and \textit{errors}, where the former uses the outcome of the latter. Diffusion fragment predictions -that is, the main diagonal of the decentralized graph signal \textit{predictions}- are stored in $u$\textit{.prediction}$=\text{frag}_v(G_{\theta[u]}, X[u], A[u])$, where $G_\theta$ is the FDiff-scale architecture. There are two hyper-parameters to be selected before deployment: $\beta\in[0,1)$ that determines the diffusion rate and $s$ that trades-off errors and predictions. Importantly, given linear or faster convergence rates for base classifier updates, Theorem~\ref{thm:convergence} yields linear convergence in distribution for \textit{errors} and hence for the in-code variable \textit{combined} of each node. Therefore, from the same theorem, \textit{predictions} also converges linearly in distribution.

\begin{figure*}[!htpb]
    \centering
    \frame{\includegraphics[width=0.9\textwidth,clip]{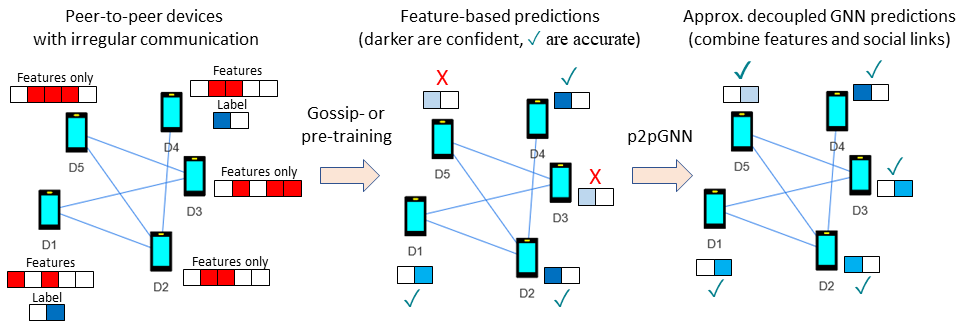}}
    \caption{p2pGNN helps peer-to-peer devices classify themselves by improving local feature-based classifiers with fragments of decentralized graph diffusion that approximate the FDiff-scale decoupled GNN.}
    \label{fig:overview}
\end{figure*}

\begin{algorithm}[htpb]
\begin{algorithmic}
\Procedure{initialize}{base\_prediction , target} 
\State $u$.predictions $\leftarrow $ Map()
\State $u$.errors $\leftarrow $ Map()
\State $u$.target $\leftarrow$ target
\State $u$.\textsc{update}(base\_prediction)
\EndProcedure
\Procedure{update}{base\_prediction}
\State $u$.base\_prediction $\leftarrow$ base\_prediction 
\State $u$.prediction $\leftarrow$ prediction
\If{$\|u.\text{target}\|\neq 0$}
\State $u$.error $\leftarrow$ (prediction $-$ target)
\EndIf
\EndProcedure
\Procedure{receive}{$v$.id, message}
\State message $\leftarrow$ $u$.\textsc{send}($v$.id)
\State $u$.\textsc{acknowledge}($v$.id, message)
\State\Return message
\EndProcedure
\Procedure{send}{$v$.id}
\State\Return $\frac{u.\text{prediction}}{|u.\text{predictons}|^{0.5}}$, $\frac{u.\text{error}}{|\text{$u$.errors}|}$
\EndProcedure
\Procedure{acknowledge}{$v$.id, message}
\State prediction, error $\leftarrow$ message
\State $u$.predictions[$u$.id] $\leftarrow$ $u$.prediction
\State $u$.errors[$u$.id] $\leftarrow$ $u$.error
\State $u$.predictions[$v$.id] $\leftarrow$ prediction
\State $u$.errors[$v$.id] $\leftarrow$ error
\If{$\|u.\text{target}\|=0$}
\State $u$.error $\leftarrow$ $\sum u$.errors.values()
\State combined $\leftarrow$ $u$.base\_prediction$+s\cdot u$.error
\Else
\State combined $\leftarrow$ $u$.base\_prediction
\EndIf
\State $u$.prediction $\leftarrow$ $(1-\beta)\cdot$combined
\State $\quad\quad+\frac{\beta}{|\text{$u$.predictions}|^{0.5}}$$\sum $$u$.predictions.values()
\EndProcedure
\end{algorithmic}
\caption{p2pGNN operations at devices $u$}\label{procedures}
\end{algorithm}

\section{Experiments}
\label{experiments}

\subsection{Datasets and Simulation}
To compare the ability of peer-to-peer learning algorithms to make accurate predictions, we experiment on three datasets that are often used to assess the quality of GNNs \cite{shchur2018pitfalls}; the Citeseer \cite{namata2012query}, Cora \cite{sen2008collective} and Pubmed \cite{namata2012query} social graphs. Pre-processed versions of these are retrieved from the programming interface of the publicly available Deep Graph Library \cite{wang2019deep} and comprise node features and class labels. They also come along training-validation-test sets commonly used in GNN literature experiments and which we also use. 
\par
The selected datasets comprise 
social links between their nodes and textual node feature data. We consider them representative samples of complex social networks with node features, even if they comprise document instead of human or sensor nodes. Their quantitative characteristics are summarized in Table~\ref{datasets}. In practice, the class labels of training and validation nodes would have been manually provided by respective devices (e.g. submitted by their users) and would form the ground truth to train base models.

\begin{table}[htpb]
    \centering\setlength\tabcolsep{3pt} 
    \caption{Dataset details}
    \begin{tabular}{l r r r r r r r}
         \textbf{Dataset} &  \textbf{Nodes} & \textbf{Links} & \textbf{Features} & \textbf{Labels} & \textbf{Train} & \textbf{Valid} & \textbf{Test}\\
         \hline
         Citeseer & 3,327 & 9,228 & 3,703 & 6 & 120 & 500 & 1,000\\
         Cora & 2,708 & 10,556 & 1,433 & 7 & 140 & 500 & 1,000\\
         Pubmed & 19,717 & 88,651 & 500 & 3 & 60 & 500 & 1,000\\
    \end{tabular}
    \label{datasets}
\end{table}

We use these datasets to simulate peer-to-peer networks with the same nodes and links as in the dataset graphs and fixed probabilities for communication through links at each time step, uniformly sampled from the range $[0,0.1]$. To speed up experiments, we further force nodes to engage in only one communication at each time step by randomly determining which edges to ignore when conflicts arise; we thus use threading to parallelize experiments by distributing time step computations between available CPUs (this is independent of our decentralized setting and its only purpose is to speed-up simulations). 
\par
Finally, we measure classification accuracy of test labels after 1000 time steps (all algorithms converge well within that number) and report its average across five experiment repetitions. Similar results are obtained for communication rates sampled from different range intervals. Experiments are available online\footnote{\url{https://github.com/MKLab-ITI/decentralized-gnn} Apache License 2.0.} and were conducted on a machine running Python 3.6 with 64GB RAM (they require at least 12GB available to run) and 32x1.80GHz CPUs.

\begin{table*}[!htbp]
\small
    \centering
    \caption{Comparing the accuracy of different types and training schemes of base algorithms and their combination with the diffusion of p2pGNN. Accuracy is computed after $1000$ time steps and averaged across $5$ peer-to-peer simulation runs.}
    \begin{tabular}{l | c c c | c c c | c c c}
    ~ & \multicolumn{3}{c|}{\textbf{Base}} & \multicolumn{3}{c}{\textbf{p2pGNN}} & \multicolumn{3}{c}{\textbf{Fully Centralized GNN}} \\
    ~ &  {Citeseer} & {Cora} & {Pubmed} & {Citeseer} & {Cora} & {Pubmed}& {Citeseer} & {Cora} & {Pubmed}\\
    \hline
    \multicolumn{1}{c}{} & \multicolumn{9}{c}{\textbf{Pre-trained}}\\
    MLP & 52.3\% & 54.9\% & 70.9\% & {}{67.8\%} & {}{81.5\%} & {}{76.0\%} & 69.0\% & 84.0\% & 81.2\%\\
    LR & 59.4\% & 58.7\% & 72.2\% & {}{70.5\%} & {}{82.0\%} & {}{77.3\%} & 70.3\% & 85.7\% & 81.5\%\\
    \multicolumn{1}{c}{} & \multicolumn{9}{c}{\textbf{Gossip}}\\
    MLP & {}{63.1\%} & 66.3\% & 74.9\% & 61.3\% & {}{80.8\%} & {}{78.0\%} & 69.0\% & 84.0\% & 81.2\%\\
    LR & 61.8\% & 79.9\% & 78.7\% & {}{61.4\%} & {}{80.8\%} & {}{78.7\%} & 70.3\% & 85.7\% & 81.5\%\\
    Labels & 15.9\% & 11.6\% & 22.0\% & {}{61.1\%} & {}{80.8\%} & {}{71.5\%} & 61.5\% & 78.9\% & 78.6\%
    \end{tabular}
    \label{tab:acc60}
\end{table*}

\subsection{Base Classifiers}
Experiments span the following three base classifiers. These cover a wide breadth of machine learning sophistication, from no learning to neural networks. Hence, we expect usage of other base classifiers to exhibit similar qualitative outcomes to those we report later on.

\begin{itemize}
\item\textit{MLP} -- A multilayer perceptron is often employed by GNNs \cite{klicpera2018predict,huang2020combining}. This consists of a dense two-layer architecture starting from a transformation of node features into $64$-dimensional representations activating ReLU outputs and a dense transformation of the latter whose softmax aims to predict one-hot encodings of labels.
\item\textit{LR} -- A simple multilabel logistic regression classifier whose softmax aims to predict one-hot encodings of classification labels. 
\item\textit{Label} -- Classification that repeats training node labels. If no diffusion is performed, this outputs random predictions for test nodes.
\end{itemize}
\par
MLP and LR are trained towards minimizing the cross-entropy loss of known node labels with Adam optimizers \cite{kingma2014adam,bock2018improvement}. We set learning rates to $0.01$, which is a value often used for training on similarly-sized datasets, and maintain the default momentum parameters proposed by the optimizer's original publication. For MLP, we use $50\%$ dropout for the dense layer to improve robustness and for all classifies we L2-regularize dense layer weights with $0.0005$ penalty. 
\par
We do not perform hyperparameter tuning, as in practice further protocols would be needed to make peer-to-peer nodes learn a common architecture optimal for a set of validation nodes. Instead, the above-described parameter values are commonly used defaults. For FDiff-scale hyperparameters, we select a personalized PageRank restart probability often used for graphs of several thousand nodes $1-\beta=0.1$ and error scale parameter $s=1$, where the latter is selected so that it theoretically satisfies a heuristic requirement of perfectly reconstructing the class labels of training nodes.

\subsection{Compared Approaches}
We experiment with the following two versions of MLP and LR classifiers, which differ with respect to whether they are pre-trained and deployed to nodes or learned via gossip averaging. In total, experiments span 2 MLP + 2 LR + Label = 7 base classifiers.
\vspace{0.3em}
~\\\textit{Pre-trained} -- Training classifier parameters in a centralized architecture over $3000$ epochs, where parameter updates of the Adam optimizer aim to maximize the cross-entropy loss of the training node set. We select the parameters at the epoch maximizing the validation node set loss, effectively tuning the number of epochs. For faster training, we perform early stopping if the validation node set loss has not decreased for $100$ epochs, which happens well within the designated maximum number of epochs, i.e. there would be no benefit or change to training time if more maximum epochs were considered. 
\par
We remind that, in practice, pre-trained classifiers can take the form of a service (e.g. a web service) that trains parameters $\theta$ based on sample data submitted by some (but not by necessarily many) devices and hosts the result. In this case, all devices $u$ query the service to obtain identical copies of the pre-trained parameters $\theta[u]=\theta$ and use these for in-device predictions and potential improvement of the latter with peer-to-peer graph diffusion. Only data of training and validation nodes are shared with the centralized service and the rest retain privacy --- hence we consider this approach partially decentralized. For ease of understanding, we assume that training has been completed before the first time step of simulated peer-to-peer communication, but in practice our approach allows linear rate (or faster) updates based on intermediate training results.
\vspace{0.3em}
~\\\textit{Gossip} -- Fully decentralized gossip averaging, where each node holds a copy of the base classifier and parameters are averaged between communicating nodes. Since no stopping criterion can be enforced, both training and validation nodes contribute to training of base classifier fragment parameters $\theta[u]$. In particular, the simulated devices corresponding to those nodes perform epoch updates on local instances of the Adam optimizer every time they are involved in a communication. During these updates, each device performs one gradient update to reduce the cross-entropy loss of its one local data sample before performing the averaging.
\par
If training data were independent and identically distributed and with many samples residing on each device, this approach could be considered a state-of-the-art baseline in terms of accuracy, as indicated by the theoretical analysis of Koloskova et al. \cite{koloskova2019decentralized} and experiment results of Niwa et al. \cite{niwa2020edge}. However, our setting of classifying devices ties at most one sample to each device and hence does not preserve these requirements. Thus, the efficacy of this practice is uncertain. We also consider the Label classifier as natively Gossip, as it does not require any centralized infrastructure.
\vspace{0.3em}
\par
For all base classifiers, we report: a) their vanilla accuracy, b) the accuracy of passing base predictions through the FDiff-scale scheme of \eqref{target diffusion}, as approximated via p2pGNN operations presented in Algorithm~\ref{procedures}, and c) the accuracy of passing the predictions of centralized counterparts through an also centralized implementation of FDiff-scale with the same hyperparameters, i.e. the last approach is fully centralized.
\par
Finally, given that training does not depend on diffusion, we perform the latter by considering both training and validation node labels as known information. That is, both types of nodes form the set $V_{train}$ of our analysis. Ideally, p2pGNN would leverage the homophilous node communications to improve base accuracy and tightly approximate fully-centralized predictions. In this case, it would become a decentralized equivalent to centralized diffusion that works under uncertain communication availability and does not expose predictive information to devices other than communicating graph neighbors.

\subsection{Results}
In Table~\ref{tab:acc60} we compare the accuracy of base algorithms vs. their augmented predictions with the decentralized p2pGNN and a fully centralized implementation of FDiff-scale. We remind that the last two schemes implement the same architecture and differ only on whether diffusion runs on peer-to-peer networks or not respectively. We can see that, in case of pre-trained base classifiers, p2pGNN successfully improves accuracy scores  by wide margins, i.e. 7\%-47\% relative increase. In fact, the improved scores closely resemble the ones of centralized diffusion, i.e. with less than 3\% relative decrease, for the Citeseer and Cora datasets. In these cases, we consider our peer-to-peer diffusion algorithm to have successfully decentralized its components. On the Pubmed dataset, centralized schemes are replicated less tightly (this also holds true for simple Label propagation), but there is still substantial improvement compared to pre-trained base classifiers.
\par
On the other hand, results are mixed for base classifiers trained via gossip averaging. Before further exploration, we remark that MLP and LR outperform their pre-trained counterparts in large part due to a combination of training with larger sets of node labels (both training and validation nodes) and ``leaking'' the graph structure into local classifier fragment parameters due to non-identically distributed node class labels. Thus, gossip training already implicitly incorporates diffusion. However, after diffusion is performed, accuracy does not reach the same levels as pre-trained base classifiers---in fact, in the Citesser and Cora datasets, homophilous parameter training reduces the diffusion of classifier fragment parameters to the diffusion of class labels. This indicates that classifier fragments tend to correlate node features with graph structure and hence additional diffusion operations are not necessarily meaningful. Characteristically, the linear nature of LR makes its base gossip-trained and p2pGNN versions near-identical. Since this issue systemically arises from gossip training shortcomings,  we leave its mitigation to future research.

\par
Overall, experiment results indicate that, in most cases, p2pGNN successfully applies GNN principles to improve base classifier accuracy. Importantly, although neighbor-based gossip training of base classifiers on both training and validation nodes outperforms models pre-trained on only training nodes (in which case validation nodes are used for early stopping), decentralized graph diffusion of the latter exhibits the highest accuracy across most combinations of datasets and base classifiers.

\subsection{Practical Exploration}
To gain an understanding of our approach's practical applicability, in Table~\ref{cost} we investigate the added communication overhead of employing decentralized graph diffusion. To do this, we serialize messages using the \textit{pickle} library \cite{van1995python} and measure the number of bytes the result takes up in-memory. This depends on the number of exchanged classifier parameters and decentralized graph signal transmisions and is fixed for each dataset.\footnote{Data compression encodings could be employed to reduce communication costs, but we expect these to maintain similar relative differences, since numerical data of few zeroes are exchanged.} In the real world, serialized messages could be sent alongside other forms of communication (e.g. social messaging) to guarantee that they reach their recipients. Alternatively, they could be exchanged whenever communication channels become available.
\par
We can see that, thanks to decoupled GNNs propagating vectors of few class label estimations and their errors, only a small overhead is added to information transmission, which lies in the order of magnitude of less than a kilobyte. In fact, this overhead can be considered negligible when compared to the communication cost of gossip training of MLP and LR base classifiers that requires 40 or more times the number of bytes. As a final note, we stress that these experiments do not capture (report as zero) communication costs for receiving pre-trained models from a central infrastructure, as this is an one-time operation. 

\begin{table}[htpb]
\small
    \centering\setlength\tabcolsep{3pt} 
    \caption{Comparing overhead in bytes (B) and kilobytes (kB---used for large overheads with rounded off decimal digits) during peer-to-peer communication between base algorithms and p2pGNN variations. The latter require less than one additional kilobyte.}\label{cost}
    \begin{tabular}{l | c c c | c c c}
    ~ & \multicolumn{3}{c|}{\textbf{Base}} & \multicolumn{3}{c}{\textbf{p2pGNN}} \\
    ~ &  {Citeseer} & {Cora} & {Pubmed} & {Citeseer} & {Cora} & {Pubmed}\\
    \hline
    \multicolumn{1}{c}{} & \multicolumn{6}{c}{\textbf{Pre-trained}}\\
    MLP & 0 & 0 & 0 & 334B & 350B & 286B\\
    LR & 0 & 0 & 0 & 334B & 350B & 286B\\
    \multicolumn{1}{c}{} & \multicolumn{6}{c}{\textbf{Gossip}}\\
    MLP & 1899kB & 738kB & 258kB & 1899kB & 738kB & 258kB\\
    LR & 178kB & 82kB & 12kB & 178kB & 82kB & 12kB\\
    Labels & 0 & 0 & 0 & 334B & 350B & 286B\\
    \end{tabular}
\end{table}

Finally, in Fig.~\ref{fig:convergence} we investigate the convergence process of p2pGNN variations in terms of predictive accuracy. To do this, we plot how accuracy evolves over times in one repetition of our experiments (that is, for a specific randomization seed) and the accuracy when communications are performed at half the rate. First, we verify that diffusion exhibits linear convergence, as accuracy values quickly approach their asymptotic limit. This is achieved within 100-200 time steps for our experiments and less than twice as many time steps when the communication rate is halved. 
\par
To understand why the product between the communication rate and the number of steps does not increase, we refer to the proof of Theorem~\ref{thm:convergence}, where the convergence rate is upper-bounded by the minimum communication rate $p_{com}$ between nodes (since the convergence rate is less than $a\sigma< p_{com}$). Thus, halving the communication rate of all edges also halves the upper stochastic bound of the convergence rate and at most doubles convergence time.

\begin{figure*}[htpb]
    \centering
    \includegraphics[width=0.9\textwidth,clip]{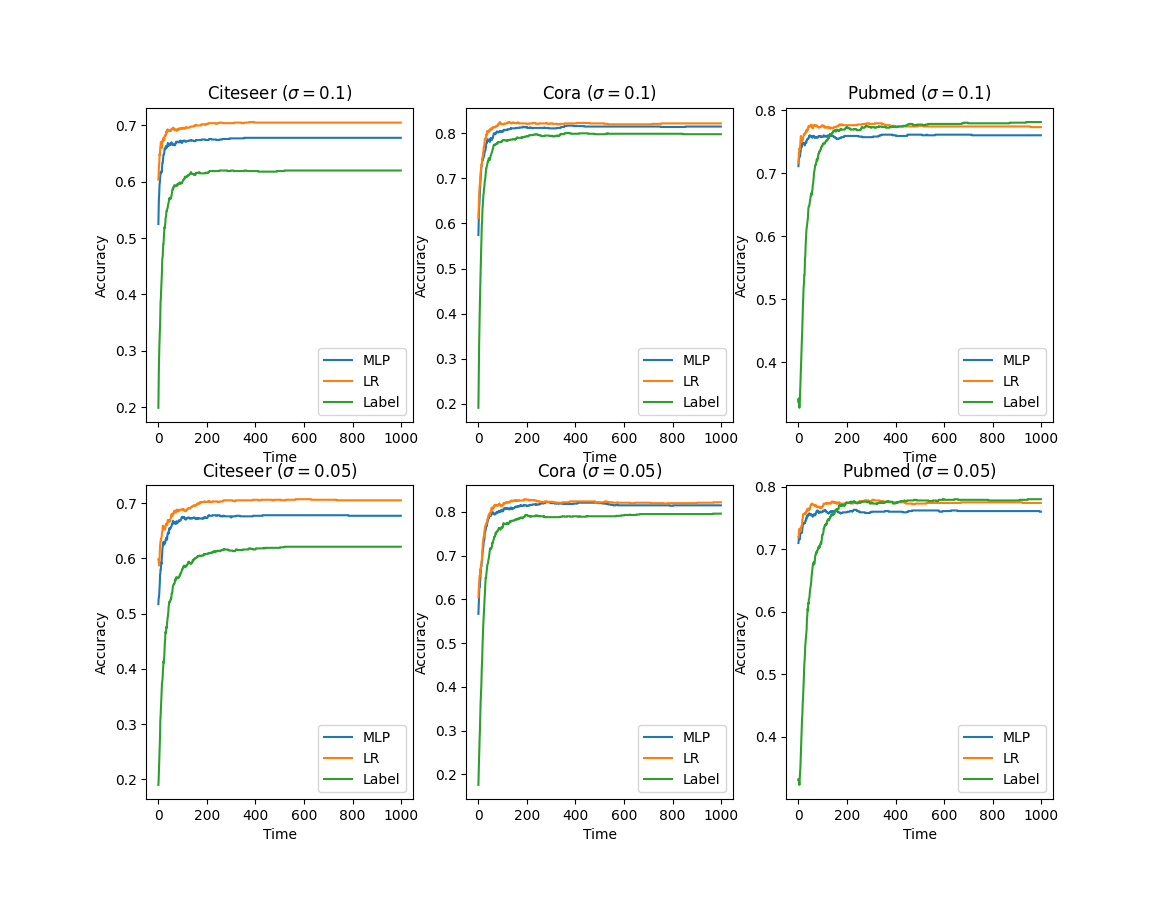}
    \caption{Accuracy convergence over $1000$ time steps of p2pGNN over all datasets for pre-trained base classifiers and label diffusion. Simulated peer-to-peer communication between neighbors takes place with rates uniformly sampled from the range $[0,\sigma_{\max}]$ where the maximum communication frequency $\sigma_{\max}$ is either 0.1 (top) or 0.05 (bottom).}
    \label{fig:convergence}
\end{figure*}

\section{Conclusions and Future Work}
In this work, we investigated the problem of letting nodes of unstructured peer-to-peer networks classify themselves under communication uncertainty and proposed that homophilous communication links can be mined with decoupled GNN diffusion to improve base classifier accuracy. We thus introduced a decentralized implementation of diffusion, called p2pGNN, whose fragments run on devices and mine network links as irregular peer-to-peer communication takes place. Theoretical analysis and experiments on three simulated peer-to-peer networks from labeled graph data showed that combining pre-trained (and often gossip-trained) base classifiers with our approach successfully improves their accuracy at comparable degrees to fully centralized decoupled graph neural networks while introducing non-intrusive communication overheads.
\par
For future work, we aim to improve gossip training to let it account for our setting's non-identically distributed spread of data samples across graph nodes, which systemically arises when each device accommodates only one sample. We are also interested in addressing privacy concerns and societal biases in our approach and explore automated hyperparameter selection.

\bibliographystyle{unsrtnat}
\bibliography{main}

\begin{IEEEbiography}[{\includegraphics[width=1in,height=1.25in,clip,keepaspectratio]{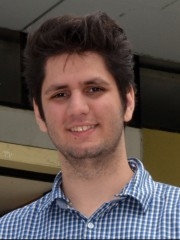}}]{Emmanouil Krasanakis} was born in Thessaloniki, Greece in 
1991. He received the diploma in electrical and computer engineering from the Aristotle University of Thessaloniki in 2015. Since 2016 he has been a Ph.D. student in the same institute and a Research Assistant in the Information Technologies Institute. He is the co-author of 4 journal articles and 9 conference papers. His research interests include graph mining and learning, decentralized learning, fairness in machine learning, formal methods, and automated software engineering.
\end{IEEEbiography}

\begin{IEEEbiography}[{\includegraphics[width=1in,height=1.25in,clip,keepaspectratio]{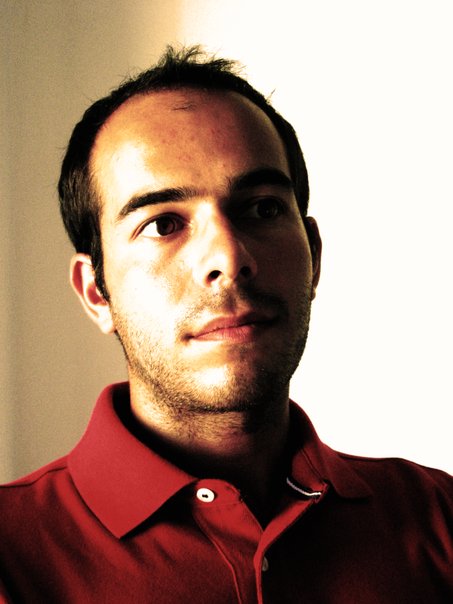}}]{Symeon Papadopoulos} was born in Thessaloniki, Greece in 1981. He received the diploma in electrical and computer engineering from the Aristotle University of Thessaloniki in 2004, the Professional Doctorate in Engineering from the Technical University of Eindhoven in 2006, the Masters in Business Administration from the Blekinge Institute of Technology in 2009 and the Ph.D. degree from the Aristotle University of Thessaloniki in 2012. He is working as Senior Researcher in the Information Technologies Institute. He is the co-author of more than 30 journal articles, 110 conference papers and 14 book chapters. His research interests include artificial intelligence for multimedia analysis, social media and big data analysis, and applications for media and journalism with a focus on online disinformation and media verification, as well as on environment, society and others. He holds 3 patents.
\par
Dr. Papadopoulos has co-organized workshops and summer schools, the most recent ones being the series of \textit{Conversations} workshops (2017-2021). He has guest edited two special issues, two books and received 3 best paper, 3 best paper candidate and 2 best demo awards.
\end{IEEEbiography}

\begin{IEEEbiography}[{\includegraphics[width=1in,height=1.25in,clip,keepaspectratio]{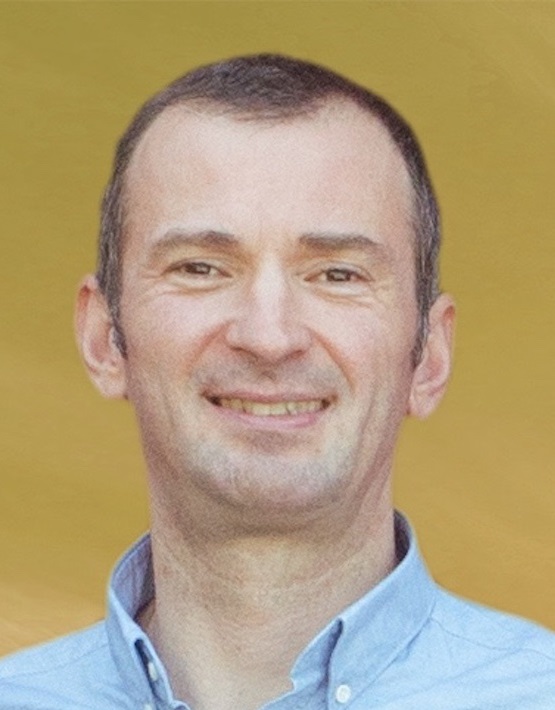}}]{Ioannis Kompatsiaris} was born in Thessaloniki in 1973. He received the B.S. degree in electrical and computer engineering from the Aristotle University of Thessaloniki in 1996 and the Ph.D. degree from the same University in 2001. He is the Director of the Information Technologies Institute and Head of the Multimedia Knowledge and Social Media Analytics Laboratory. His research interests include image and video analysis, big data and social media analytics, semantics, human computer interfaces (AR and BCI), eHealth and security applications. He is the co-author of 178 journals articles, more than 560 conference papers and 59 book chapters. He is a Senior Member of IEEE and ACM and holds 8 patents.
\par
Dr. Kompatsiaris has organized conferences, workshops and summer schools and he is an Associate Editor of the \textit{IEEE Transactions on Image Processing}. Finally, he is a member of the National Ethics and Technoethics Committee and an elected member of IEEE \textit{Image, Video and Multidimensional Signal Processing - Technical Committee (IVMSP - TC)}.
\end{IEEEbiography}

\EOD

\end{document}